\documentclass{article}
\usepackage{algorithmicx}
\usepackage{algpseudocode}
\usepackage{amsmath}
\usepackage{amssymb}
\usepackage{amsthm}
\newtheorem{thm}{Theorem}
\newtheorem*{thm*}{Theorem}
\newtheorem{defn}{Definition}
\usepackage{boxedminipage}
\usepackage{bussproofs}
\usepackage{caption}
\usepackage[scale=0.8]{geometry}
\usepackage{graphicx}
\usepackage{hyperref}
\usepackage{latexsym}
\usepackage{natbib}
\usepackage{paralist}
\usepackage{tikz}
\tikzset{every path/.append style={thick}}
\usetikzlibrary{automata,shapes,positioning,decorations.pathreplacing,calc}

\usepackage[sc]{mathpazo}
\usepackage[T1]{fontenc}

\title{Two Algorithms for Finding $k$ Shortest Paths of a Weighted Pushdown Automaton}
\author{WU Ke\\
  Department of Computer Science\\
  and CLIP Lab at the Institute for\\
  Advanced Computer Studies\\
  University of Maryland\\
  College Park, MD 20742, USA\\
\and
  Philip RESNIK\\
  Department of Linguistics\\
  and CLIP Lab at the Institute for\\
  Advanced Computer Studies\\
  University of Maryland\\
  College Park, MD 20742, USA\\
}

\begin{document}

\maketitle{}

\section{Introduction}
\label{sec:intro}

Weighted pushdown automata (WPDAs) have recently been adopted in some
applications such as machine translation
\citep{iglesias_hierarchical_2011} as a more compact alternative to
weighted finite-state automata (WFSAs) for representing a weighted set
of strings. \citet{allauzen_pushdown_2012} introduce a set of basic
algorithms for construction and inference of WPDAs, and the
corresponding implementation as an extension of the open source
finite-state transducer toolkit
OpenFst.\footnote{http://www.openfst.org/twiki/bin/view/FST/FstExtensions}

Although a shortest-path algorithm for WPDAs with bounded stack is
described in \citet{allauzen_pushdown_2012}, it does not give a
$k$-shortest-path algorithm, which finds the $k$ shortest accepting
paths of the given automaton. Other than just the single shortest
path, $k$ shortest paths are useful for many purposes such as
reranking the output in parsing \citep{collins_discriminative_2005} or
tuning feature weights in machine translation
\citep{chiang_11001_2009}. One existing work-around is to first expand
the WPDA into an equivalent WFSA and then find the $k$ shortest paths
of the WFSA using the $k$-shortest-path algorithm for WFSAs (the
\emph{expansion approach}). Since the WPDA expansion has an
exponential time and space complexity with respect to the size of the
automaton, one usually has to prune the WPDA before expansion (the
\emph{pruned expansion approach}), i.e. remove those transitions and
states that are not on any accepting path with a weight at most a
given threshold greater than the shortest distance. However, setting
an adequate threshold that neither prunes nor keeps too many states or
transitions a priori is almost impossible in practice.

In this paper, we introduce two efficient algorithms for finding the
$k$ shortest paths of a WPDA, both derived from the same weighted
deductive logic description of the execution of a WPDA using different
search strategies.

\section{Weighted pushdown automata}
\label{sec:pushdown}

\subsection{Formal definitions}
\label{sec:def}

Following \citet{allauzen_pushdown_2012}, we represent a WPDA as
directed graph with labeled and weighted arcs (transitions).

\begin{defn}
  A WPDA $M$ over a semiring $\langle \mathbb{K}, \oplus,
  \otimes, \overline{0}, \overline{1} \rangle$ is a tuple $\langle
  \Sigma, \Pi, \hat{\Pi}, Q, E, s, f \rangle$, where
  \begin{compactitem}
  \item $\Sigma$, $\Pi$ and $\hat{\Pi}$ are disjoint finite sets of
    symbols;
  \item $\Sigma$ is the alphabet of input symbols;
  \item $\Pi$ and $\hat{\Pi}$ are the alphabets of respectively
    opening and closing parentheses; there exists a bijection between
    them that pairs the parentheses; for any $a \in \Pi \cup
    \hat{\Pi}$, we represent its counterpart in the other alphabet as
    $\hat{a}$;
  \item $Q$ is a finite set of states; $s \in Q$ is the start state and
    $f \in Q$ is the final state;
  \item $E \subseteq Q \times (\Sigma \cup \Pi \cup \hat{\Pi} \cup \{
    \epsilon \}) \times \mathbb{K} \times Q$ is a finite set of
    transitions; $e = \langle p[e], i[e], w[e], n[e] \rangle \in E$
    denotes a transition from state $p[e]$ to state $n[e]$ with label
    $i[e]$ and weight $w[e]$, where $w[e] \neq \overline{0}$.
  \end{compactitem}
\end{defn}

A path $\pi$ is a sequence of transitions $\pi = e_1 e_2 \ldots e_m$,
such that $n[e_i] = p[e_{i+1}]$ for all $1 \leq i < m$. $p[\cdot]$,
$i[\cdot]$, $w[\cdot]$ and $n[\cdot]$ can all be generalized to
paths. For a given path $\pi = e_1 e_2 \ldots e_m$, define $p[\pi] =
p[e_1]$, $n[\pi] = n[e_m]$, $i[\pi] = i[e_1]i[e_2] \ldots i[e_m]$, and
$w[\pi] = w[e_1] \otimes w[e_2] \otimes \ldots \otimes w[e_m]$. Unlike
a WFSA, not all paths from $s$ to $f$ in a WPDA are accepting
paths. For a set of symbols $S$, let $c_S[\pi]$ be the substring of
$i[\pi]$ consisting of all and only the symbols from set $S$. For
example, $c_{\Pi \cup \hat{\Pi}}[\pi]$ is the substring of $i[\pi]$
consisting of all and only the opening and closing parentheses. Then,
\begin{defn}
  The Dyck language on finite parenthesis alphabets $\Pi$ and
  $\hat{\Pi}$ consists of strings of balanced parentheses. A path
  $\pi$ is balanced if $c_{\Pi \cup \hat{\Pi}}[\pi]$ belongs to
  the Dyck language on $\Pi$ and $\hat{\Pi}$.
\end{defn}
For example, when $\Pi = \{ \text{`('}, \text{`['} \}$ and
$\hat{\Pi} = \{ \text{`)'}, \text{`]'} \}$ with normal pairing by
appearance, strings such as $()$, $([()])[]$ are members of the Dyck
language while $($ or $(][)$ are not.

Finally,
\begin{defn}
  A path $\pi$ is an accepting path if and only if $p[\pi] = s$,
  $n[\pi] = f$ and $\pi$ is balanced.
\end{defn}

This representation of WPDAs is slightly different from the classical
representation of PDAs, where a stack alphabet is defined with
optional push or pop operations at each transition. Here the stack
alphabet is essentially $\Pi$ and $\hat{\Pi}$, paired by the bijection
between them. Whenever a symbol from $\Pi$ is consumed, it is
equivalent to pushing the particular symbol onto the stack in the
classical representation; and whenever a symbol from $\hat{\Pi}$ is
consumed, it is equivalent to popping a symbol off the stack and
checking if the symbol is its counterpart from $\Pi$. As discussed in
\citet{allauzen_pushdown_2012}, such representation leads to easy
adaptation of some WFSA algorithms for similar purposes on a WPDA.

Following \citet{allauzen_pushdown_2012}, we limit our effort in
finding $k$ shortest paths to WPDAs with a bounded stack in both
pushing and popping.\footnote{This definition is slightly different from
  \citet{allauzen_pushdown_2012}, which only bounds pushing.}
\begin{defn}
  A WPDA has a bounded stack if there exists an integer $K$ such that
  for any path $\pi$, the number of unmatched parenthesis in
  $c_{\Pi}[\pi]$ is no greater than $K$.
\end{defn}
Although this rules out all WPDAs with recursion, the ones found in
applications that need to find the $k$ shortest paths usually do not
have recursion \citep{iglesias_hierarchical_2011}. Thus an algorithm
that only works on WPDAs with a bounded stack is already very useful.

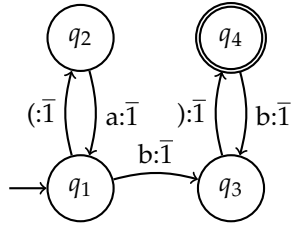
\begin{figure}
  \centering
  \begin{tikzpicture}[node distance=2cm,bend angle=15]
    \node[state,initial,initial text={}](s1){$q_1$};
    \node[state](s2)[above of=s1]{$q_2$};
    \node[state](s3)[right of=s1]{$q_3$};
    \node[state,accepting](s4)[above of=s3]{$q_4$};
    \path[->]
    (s1) edge [bend left] node[auto] {(:$\overline{1}$} (s2)
         edge [bend left] node[auto] {b:$\overline{1}$} (s3)
    (s2) edge [bend left] node[auto] {a:$\overline{1}$} (s1)
    (s3) edge [bend left] node[auto] {):$\overline{1}$} (s4)
    (s4) edge [bend left] node[auto] {b:$\overline{1}$} (s3);
  \end{tikzpicture}
  \caption{A WPDA of $\{a^nb^n | n > 0 \}$}
  \label{fig:pda}
\end{figure}

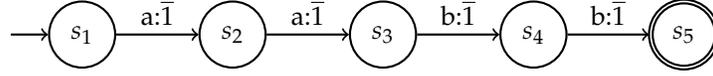
\begin{figure}
  \centering
  \begin{tikzpicture}[node distance=2cm,bend angle=15]
    \node[state,initial,initial text={}](s1){$s_1$};
    \node[state](s2)[right of=s1]{$s_2$};
    \node[state](s3)[right of=s2]{$s_3$};
    \node[state](s4)[right of=s3]{$s_4$};
    \node[state,accepting](s5)[right of=s4]{$s_5$};
    \path[->]
    (s1) edge node[auto] {a:$\overline{1}$} (s2)
    (s2) edge node[auto] {a:$\overline{1}$} (s3)
    (s3) edge node[auto] {b:$\overline{1}$} (s4)
    (s4) edge node[auto] {b:$\overline{1}$} (s5);
  \end{tikzpicture}
  \caption{Input string ``aabb'' encoded as a WFSA}
  \label{fig:fsa}
\end{figure}

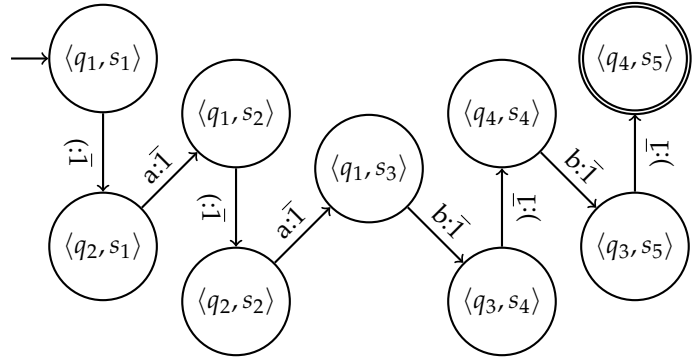
\begin{figure}
  \centering
  \begin{tikzpicture}[node distance=2.5cm,bend angle=15]
    \node[state,initial,initial text={}](s1){$\langle q_1, s_1 \rangle$};
    \node[state](s2)[below of=s1]{$\langle q_2, s_1 \rangle$};
    \node[state](s3)[above right of=s2]{$\langle q_1, s_2 \rangle$};
    \node[state](s4)[below of=s3]{$\langle q_2, s_2 \rangle$};
    \node[state](s5)[above right of=s4]{$\langle q_1, s_3 \rangle$};
    \node[state](s6)[below right of=s5]{$\langle q_3, s_4 \rangle$};
    \node[state](s7)[above of=s6]{$\langle q_4, s_4 \rangle$};
    \node[state](s8)[below right of=s7]{$\langle q_3, s_5 \rangle$};
    \node[state,accepting](s9)[above of=s8]{$\langle q_4, s_5 \rangle$};
    \path[->]
    (s1) edge node[below,sloped] {(:$\overline{1}$} (s2)
    (s2) edge node[above,sloped] {a:$\overline{1}$} (s3)
    (s3) edge node[below,sloped] {(:$\overline{1}$} (s4)
    (s4) edge node[above,sloped] {a:$\overline{1}$} (s5)
    (s5) edge node[above,sloped] {b:$\overline{1}$} (s6)
    (s6) edge node[below,sloped] {):$\overline{1}$} (s7)
    (s7) edge node[above,sloped] {b:$\overline{1}$} (s8)
    (s8) edge node[below,sloped] {):$\overline{1}$} (s9);
  \end{tikzpicture}
  \caption{The result of intersection}
  \label{fig:fsa-pda}
\end{figure}

\citet{allauzen_pushdown_2012} give a general algorithm for converting
a context free grammar into an equivalent WPDA. Figure \ref{fig:pda}
is an example WPDA representing the classical context free language
$\{a^nb^n | n > 0\}$ constructed following their algorithm. It is easy
to see this WPDA does not have a bounded stack. However, considering
this as a ``grammar'', one can then ``parse'' strings with the grammar
by encoding the input as a WFSA and intersecting the WPDA with it. For
example, Figure \ref{fig:fsa-pda} is the result of intersecting Figure
\ref{fig:fsa} with Figure \ref{fig:pda}, which now has a bounded
stack.

\subsection{Automata execution as weighted deduction}
\label{sec:deduction}

A deductive logic defines a space of weighted items, some of which are
axioms or goals (items to prove), and a set of inference rules of the
form,
\begin{center}
  \begin{prooftree}
    \AxiomC{$A_1:w_1$}
    \AxiomC{$A_2:w_2$}
    \AxiomC{$\ldots$}
    \AxiomC{$A_m:w_m$}
    \RightLabel{\quad $\phi$}
    \QuaternaryInfC{$B:g(w_1, w_2, \ldots, w_m)$}
  \end{prooftree}
\end{center}
which means if items $A_1, A_2, \ldots, A_m$ are provable
respectively with weights $w_1, w_2, \ldots, w_m$, then item $B$ is
also provable with weight $g(w_1, w_2, \ldots, w_m)$ given the side
condition $\phi$ is satisfied. We also call $B$ proved this way
\emph{an instantiation of $B$ with weight $g(w_1, w_2, \ldots,
  w_m)$}. This style of system has been commonly used to express
parsing strategies since \citet{shieber_principles_1995}.

The execution of a WPDA $M$ can be described using the following
weighted deductive logic $\mathcal{L}_M$.

\begin{compactitem}
\item The items are of the form $q_1 \leadsto q_2$, where $q_1, q_2
  \in Q$. An instantiation $q_1 \leadsto q_2 : u$ for some $u \in
  \mathbb{K}$ intuitively means there is a balanced path from $q_1$ to
  $q_2$ with weight $u$.
\item Axioms are
  \begin{prooftree}
    \AxiomC{}
    \RightLabel{\quad $q = s$ or there exists $e \in E$ such that $n[e] = q$ and $i[e] \in \Pi$}
    \UnaryInfC{$q \leadsto q : \overline{1}$}
  \end{prooftree}
  Furthermore, we call any state $q$ an \emph{entering state} if $q
  \leadsto q : \overline{1}$ is an axiom .
\item There are two inference rules,
  \begin{compactenum}
  \item Scan
    \begin{prooftree}
      \AxiomC{$q \leadsto p[e] : u$}
      \RightLabel{\quad $e \in E$ such that $i[e] \in \Sigma \cup \{ \epsilon \}$}
      \UnaryInfC{$q \leadsto n[e] : u \otimes w[e]$}
    \end{prooftree}
  \item Complete
    \begin{prooftree}
      \AxiomC{$q \leadsto p[e_1] : u_1$}
      \AxiomC{$n[e_1] \leadsto p[e_2] : u_2$}
      \RightLabel{\quad $e_1, e_2 \in E$ such that $i[e_1] \in \Pi, i[e_2] \in \hat{\Pi}, i[e_1] = \hat{i}[e_2]$}
      \BinaryInfC{$q \leadsto n[e_2] : u_1 \otimes w[e_1] \otimes u_2 \otimes w[e_2]$}
    \end{prooftree}
  \end{compactenum}
\item The only goal item is
  \[s \leadsto f\]
\end{compactitem}

Any valid proof forms a tree that induces a path. The induced path can
be obtained by reading off the transitions in side conditions through
a left-to-right post-order traversal of the proof tree. Take the WPDA
in Figure \ref{fig:pda} for example; Figure \ref{fig:proof} is a proof
of the accepting path of the string ``aabb''. The accepting path is
thus $(7)(1)(4)(2)(3)(5)(6)(8)$, i.e. $q_1 \xrightarrow{(} q_2
\xrightarrow{a} q_1 \xrightarrow{(} q_2 \xrightarrow{a} q_1
\xrightarrow{b} q_3 \xrightarrow{)} q_4 \xrightarrow{b} q_3
\xrightarrow{)} q_4$.

\begin{figure}
  \centering
\begin{prooftree}
    \AxiomC{}
  \UnaryInfC{$q_1 \leadsto q_1 : \overline{1}$}
          \AxiomC{}
        \UnaryInfC{$q_2 \leadsto q_2 : \overline{1}$}
        \RightLabel{$q_2 \xrightarrow{a} q_1 : \overline{1} \ (1)$}
      \UnaryInfC{$q_2 \leadsto q_1 : \overline{1}$}
            \AxiomC{}
          \UnaryInfC{$q_2 \leadsto q_2 : \overline{1}$}
          \RightLabel{$q_2 \xrightarrow{a} q_1 : \overline{1} \ (2)$}
        \UnaryInfC{$q_2 \leadsto q_1 : \overline{1}$}
        \RightLabel{$q_1 \xrightarrow{b} q_3 : \overline{1} \ (3)$}
      \UnaryInfC{$q_2 \leadsto q_3 : \overline{1}$}
      \RightLabel{$q_1 \xrightarrow{(} q_2 : \overline{1} \ (4)$, $q_3 \xrightarrow{)} q_4 : \overline{1} \ (5)$}
    \BinaryInfC{$q_2 \leadsto q_4 : \overline{1}$}
    \RightLabel{$q_4 \xrightarrow{b} q_3 : \overline{1} \ (6)$}
  \UnaryInfC{$q_2 \leadsto q_3 : \overline{1}$}
  \RightLabel{$q_1 \xrightarrow{(} q_2 : \overline{1} \ (7)$, $q_3 \xrightarrow{)} q_4 : \overline{1} \ (8)$}
\BinaryInfC{$q_1 \leadsto q_4 : \overline{1}$}
\end{prooftree}
  \caption{A proof of the accepting path of ``aabb''}
  \label{fig:proof}
\end{figure}

One can easily prove the following by induction for any WPDA $M$ (see
the appendix),\footnote{Note especially that a bounded stack is not
  required.}

\begin{thm}[Soundness]
  Any valid proof of an instantiation $q_1 \leadsto q_2 : u$ in $\mathcal{L}_M$
  induces a balanced path from $q_1$ to $q_2$ with weight $u$ in $M$.
\end{thm}

\begin{thm}[Completeness]
  Any balanced path from an entering state $q_1$ to some state $q_2$
  with weight $u$ in $M$ has a valid proof of an instantiation $q_1
  \leadsto q_2 : u$ in $\mathcal{L}_M$ whose induced path is that
  path.
\end{thm}

\begin{thm}[In-ambiguity]
  Any balanced path from an entering state in $M$ has a unique proof
  in $\mathcal{L}_M$.\footnote{Up to the tree structure with side
    conditions.}
\end{thm}

The three properties together essentially state that there is a
one-to-one correspondence between proofs of goal items in
$\mathcal{L}_M$ and accepting paths in $M$.

\subsection{The $k$-shortest-path problem}
\label{sec:problem}

The $k$-shortest-path problem on a WPDA $M$ with a bounded stack is to
find $k$ accepting paths from $M$ with the smallest weights with
respect to the natural ordering of $M$'s weight semiring
$\mathbb{K}$.\footnote{In the rest of this paper, we always assume the
  WPDA $M$ has a bounded stack.}

The natural ordering $\leq \subseteq \mathbb{K} \times \mathbb{K}$ is
defined as
\begin{defn}
  For any $a, b \in \mathbb{K}$, $a \leq b$ if and only if $a \oplus b
  = a$.
\end{defn}

For the problem to be well-defined, the natural ordering also has to
be total, which is equivalent to requiring the $\oplus$ operator to
have the following \emph{path property}: for any $a, b \in
\mathbb{K}$, $a \oplus b = a$ or $a \oplus b = b$. An example meeting
these conditions is the tropical semiring $\langle \mathbb{R} \cup \{
\infty \}, \min, +, \infty, 0 \rangle$, one of the most commonly used
as weights in parsing and machine translation. Its natural ordering is
simply the ordering of real numbers and infinity.

\section{Computing the Shortest Distance}
\label{sec:shortest}

One of the benefits of the above weighted deduction representation is
that many properties can be computed by carrying out the deductions in
a uniform style. As a starting point, we are interested in finding the
smallest-weight instantiation of some item $q_1 \leadsto q_2$. For
reasons which will become clear later, we call the weight of that
instantiation the \emph{inside weight} of $q_1 \leadsto q_2$. Let $R$
be the set of all instantiations of provable items. Because of the
path property, computing the inside weight of $q_1 \leadsto q_2$ is
equivalent to computing
\[
\alpha(q_1 \leadsto q_2) = \bigoplus_{\{u | q_1 \leadsto q_2 : u \in R\}} u
\]

The sum can be further grouped by the last step taken in a proof of
$q_1 \leadsto q_2 : u$. Define $A(q_1 \leadsto q_2)$ to be the
following,
\[
A(q_1 \leadsto q_2) = \left\{
  \begin{array}{ll}
    \overline{1} & q_1 \leadsto q_2 \text{ is an axiom} \\
    \overline{0} & \text{otherwise}
  \end{array}
  \right.
\]

Define $S_{q_1 \leadsto q_2} \subseteq E$ be the set of ``last steps
taken'' to prove $q_1 \leadsto q_2$ with a Scan, i.e. $e$ is in
$S_{q_1 \leadsto q_2}$ if and only if some instantiation $q_1 \leadsto
p[e]:u$ with $e$ as the side condition can prove $q_1 \leadsto q_2$
with the Scan rule. Similarly, define $C_{q_1 \leadsto q_2} \subseteq
E \times E$ be the set of ``last steps taken'' to prove $q_1 \leadsto
q_2$ with a Complete, i.e. $\langle e_1, e_2 \rangle$ is in $C_{q_1
  \leadsto q_2}$ if and only if some instantiations $q_1 \leadsto
p[e_1]:u_1$ and $n[e_1] \leadsto p[e_2]:u_2$ can prove $q_1 \leadsto
q_2$ with the Complete rule. Then, $\alpha(q_1 \leadsto q_2)$ can be
rewritten as
\begin{align*}
  \alpha(q_1 \leadsto q_2) = & A(q_1 \leadsto q_2) \oplus
  \left(
    \bigoplus_{e \in S_{q_1 \leadsto q_2}} \bigoplus_{\{u | q_1 \leadsto p[e]:u \in R\}} u \otimes w[e]
  \right) \oplus \\
  & \left(
    \bigoplus_{\langle e_1, e_2 \rangle \in C_{q_1 \leadsto q_2}} \bigoplus_{\{u_1 | q_1 \leadsto p[e_1]:u_1 \in R\}} \bigoplus_{\{u_2 | n[e_1] \leadsto p[e_2]:u_2 \in R\}} u_1 \otimes w[e_1] \otimes u_2 \otimes w[e_2]
  \right) \\
  = & A(q_1 \leadsto q_2) \oplus
  \left(
    \bigoplus_{e \in S_{q_1 \leadsto q_2}} \alpha(q_1 \leadsto p[e]) \otimes w[e]
  \right) \oplus \\
  & \left(
    \bigoplus_{\langle e_1, e_2 \rangle \in C_{q_1 \leadsto q_2}} \alpha(q_1 \leadsto p[e_1]) \otimes w[e_1] \otimes \alpha(n[e_1] \leadsto p[e_2]) \otimes w[e_2]
  \right)
\end{align*}

This recursive formulation allows us to compute the shortest distance
of an item using the shortest distance of its component
sub-items. When the WPDA $M$ has a bounded stack, one can easily
derive an algorithm that computes the shortest distance using
$\mathcal{L}_M$. Figure \ref{fig:inside-alg} is a simple example of
such an algorithm. This algorithm carries out a standard agenda-based
reasoning with the relaxation technique
\citep{cormen_relaxation_2009}, where $Q$ is the agenda. The map
$\alpha$ maintains the current estimate of each proven item's inside
weight. Lines 4-7 seed the axioms as the starting point of
reasoning. Then lines 8-26 try to prove new items by applying the Scan
rule (lines 12-13) and the Complete rule (lines 14-24). Any item that
is newly proven or proven with a smaller weight is added back to the
agenda in the $Relax$ function.

The above algorithm is conveniently derived from the weighted
deduction system using standard techniques. Nevertheless, there are
other strategies that can also be used; for example, the shortest path
algorithm in \citet{allauzen_pushdown_2012} is essentially computing
the inside weights with a multi-agenda strategy.

\begin{figure}
  \centering
  \begin{boxedminipage}{0.9\textwidth}
    \begin{algorithmic}[5]
      \Function {$Inside$}{}
      \State {$\alpha \gets \text{ empty map}$}
      \State {$Q \gets \text{ empty queue}$}
      \ForAll {entering state $q$}
        \State {$Push(q \leadsto q, Q)$}
        \State {$\alpha[q \leadsto q] \gets \overline{1}$}
      \EndFor
      \While {$Q$ is not empty}
        \State {$q_1 \leadsto q_2 \gets Pop(Q)$}
        \State {$u \gets \alpha[q_1 \leadsto q_2]$}
        \ForAll {transition $e$ such that $p[e] = q_2$}
          \If {$i[e] \in \Sigma \cup \{ \epsilon \}$}  \Comment {Scan}
            \State{$Relax(q_1 \leadsto n[e], u \otimes w[e])$}
          \ElsIf {$i[e] \in \Pi$} \Comment {Complete; as the left antecedent}
            \ForAll {$e'$ such that $i[e'] = \hat{i}[e]$ and $n[e] \leadsto p[e']$ is in $\alpha$}
              \State{$Relax(q_1 \leadsto n[e'], u \otimes w[e] \otimes \alpha[n[e] \leadsto p[e']] \otimes w[e'])$}
            \EndFor
          \ElsIf {$i[e] \in \hat{\Pi}$} \Comment {Complete; as the right antecedent}
            \ForAll {$e'$ such that $i[e'] = \hat{i}[e]$ and $n[e'] = q_1$}
              \ForAll {$q_3$ such that $q_3 \leadsto p[e']$ is in $\alpha$}
                \State{$Relax(q_3 \leadsto n[e], \alpha[q_3 \leadsto p[e']] \otimes w[e'] \otimes u \otimes w[e])$}
              \EndFor
            \EndFor
          \EndIf
        \EndFor
      \EndWhile
      \EndFunction
      \State{}
      \Function{$Relax$}{$q_1 \leadsto q_2, w$}
      \If {$q_1 \leadsto q_2$ is in $\alpha$}
        \State {$u \gets \alpha[q_1 \leadsto q_2] \oplus w$}
        \If {$u \neq \alpha[q_1 \leadsto q_2]$}
          \State {$\alpha[q_1 \leadsto q_2] \gets u$}
          \State {$Push(q_1 \leadsto q_2, Q)$ if $q_1 \leadsto q_2$ not already in $Q$}
        \EndIf
      \Else
        \State {$\alpha[q_1 \leadsto q_2] \gets u$}
        \State {$Push(q_1 \leadsto q_2, Q)$ if $q_1 \leadsto q_2$ not already in $Q$}
      \EndIf
      \EndFunction
    \end{algorithmic}
  \end{boxedminipage}
  \caption{A simple Inside algorithm}
  \label{fig:inside-alg}
\end{figure}

\section{Algorithm 1}
\label{sec:algorithm-1}


Having discussed the shortest distance problem in a WPDA, we now move
on to the $k$-shortest-path problem. The key idea of our first
algorithm is similar to the A* $k$-best parsing algorithm in
\citet{pauls_k-best_2009}. As we have shown in Section
\ref{sec:deduction}, similar to parsing, the execution of a WPDA can
be described as a weighted deductive logic. The generalized A* search
algorithm from \citet{felzenszwalb_generalized_2007} can then be
applied with a monotonic and admissible heuristic function to find the
$k$ instantiations of the goal item with smallest weights, from which
we get the $k$ shortest paths. The outside weight of items can be
defined with similar meanings to parsing and used as an exact
heuristic. Another, inexact heuristic will also be discussed, which
will eventually lead to our second algorithm.

\subsection{A* search on a deductive logic}
\label{sec:a-star}

\citet{felzenszwalb_generalized_2007} introduce the generalized A*
search algorithm on a deductive logic. Although the original algorithm
assumes the weights are from a positive tropical semiring, this is not
a necessary requirement in our problem, as we show next.

Similar to the original A* algorithm on graphs \citep{hart1968formal},
we need a heuristic function $H$ to estimate the final weight
continuing from the current search state (an instantiation in this
case) to the closest goal item. More formally, for a weighted logic
$\mathcal{L}$ with (unweighted) item space $I$ on semiring $\langle
\mathbb{K}, \oplus, \otimes, \overline{0}, \overline{1} \rangle$, a
heuristic function $H: \langle I, \mathbb{K} \rangle \to \mathbb{K}$
is any function satisfying the following,
\begin{description}
\item[Admissibility] For any provable instantiation of the goal item
  $G:w$,
  \[
  H(G:w) = w
  \]
\item[Monotonicity] For any provable instantiations $A_1:w_1, A_2:w_2,
  \ldots, A_m:w_m$ and an inference rule
  \begin{prooftree}
    \AxiomC{$A_1:w_1$}
    \AxiomC{$A_2:w_2$}
    \AxiomC{$\ldots$}
    \AxiomC{$A_m:w_m$}
    \QuaternaryInfC{$B:g(w_1, w_2, \ldots, w_m)$}
  \end{prooftree}
  and $1 \leq i \leq n$,
  \[
  H(A_i:w_i) \leq H(B:g(w_1, w_2, \ldots, w_m))
  \]
  where $\leq$ is the natural ordering of the
  semiring.
\end{description}
With such an $H$, the A* algorithm on a deductive logic can then be
described as in Figure \ref{fig:alg}.

\begin{figure}
  \centering
  \begin{boxedminipage}{0.9\textwidth}
    \begin{algorithmic}[5]
      \State {$S \gets \text{ empty set of proven instantiations}$}
      \State {$Q \gets \text{ empty min-priority queue}$}
      \ForAll {axiom $A:u$}
        \State {$Push(A:u, Q)$ with priority $H(A:u)$}
      \EndFor
      \While {$Q$ is not empty}
        \State {$A:u \gets Pop(Q)$}
        \If {$A:u$ is a goal item}
          \State Output $A:u$
        \EndIf
        \State {Add $A:u$ to $S$}
        \ForAll {new instantiation $B:v$ proveable using $A:u$ and any member of $S$}
          \State {$Push(B:v, Q)$ with priority $H(B:v)$}
        \EndFor
      \EndWhile
    \end{algorithmic}
  \end{boxedminipage}
  \caption{The generalized A* algorithm}
  \label{fig:alg}
\end{figure}

Similar to the original A* algorithm, the following property holds for
the generalized A* algorithm as well:
\begin{thm}
  If a monotonic $H$ is used, the generalized A* algorithm pops
  instantiations in increasing order of their $H$ value.
\end{thm}

The proof of the tropical semiring case can be found in
\citet{felzenszwalb_generalized_2007}. We include the proof simply to
show this is the case with any monotonic heuristic function and any
semiring with the path property; not just the tropical semiring.

\begin{proof}
  Suppose some instantiation is not popped in order of the $H$
  value. Let the instantiations popped in order be $A_1:w_1, A_2:w_2,
  \ldots$ and let $i$ be the smallest index such that
  $H(A_{i-1}:w_{i-1}) > H(A_i:w_i)$. Right before $A_{i-1}:w_{i-1}$ is
  popped, $A_i:w_i$ cannot be inside $Q$, otherwise it will be popped
  instead. This means $A_i:w_i$ is added into $Q$ after popping
  $A_{i-1}:w_{i-1}$ by applying some inference rule with
  $A_{i-1}:w_{i-1}$. The application is of the form
  \begin{prooftree}
    \AxiomC{$\ldots$}
    \AxiomC{$A_{i-1}:w_{i-1}$}
    \AxiomC{$\ldots$}
    \TrinaryInfC{$A_{i}:g(\ldots, w_{i-1}, \ldots)$}
  \end{prooftree}
  Because $H$ is monotonic,
  \[
  H(A_{i-1}:w_{i-1}) \leq H(A_i:g(\ldots, w_{i-1}, \ldots)) = H(A_i:w_i)
  \]
  This contradicts the assumption $H(A_{i-1}:w_{i-1}) > H(A_i:w_i)$.
\end{proof}

If $H$ is also admissible, then for any instantiation of a goal item
$G:w$, $H(G:w)$ is just $w$. Thus such instantiations are popped in
increasing order of their weights and the first $k$ such
instantiations popped are the ones with the smallest weight.

\subsection{Outside weight as an exact heuristic}

For a given instantiation $q_1 \leadsto q_2 : u$, we want the
heuristic to tell us the weight of the shortest accepting path
continuing from this instantiation. Such a heuristic is trivially
monotonic and admissible. Let $\pi$ be the path induced by $q_1
\leadsto q_2 : u$, and define
\[
H_1(q_1 \leadsto q_2 : u) = \bigoplus_{\mu, \nu} w[\mu] \otimes u \otimes w[\nu]
\]
where the sum is over all pairs of prefixes and suffixes of
transitions such that $\mu\pi\nu$ forms an accepting path. When the
semiring is commutative, the heuristic has a simple form. Define
$\beta(q_1 \leadsto q_2) = \bigoplus_{\mu, \nu} w[\mu] \otimes
w[\nu]$; then
\[
H_1(q_1 \leadsto q_2 : u) = \beta(q_1 \leadsto q_2) \otimes u
\]

\begin{figure}
  \centering
  \begin{tikzpicture}[node distance=5em]
    \node[state,initial,initial text={}](s){$s$};
    \node[state](q0)[right of=s]{};
    \node[state](q1)[right of=q0]{$q_1$};
    \node[state](q2)[right of=q1]{$q_2$};
    \node[state](q3)[right of=q2]{};
    \node[state](q4)[right of=q3]{};
    \node[state,accepting](f)[right of=q4]{$f$};

    \path[->]
    (s) edge (q0)
    (q0) edge node[above] {(} (q1)
    (q1) edge (q2)
    (q2) edge (q3)
    (q3) edge node[above] {)} (q4)
    (q4) edge (f);

    \draw ($(q1.north) + (0, 1ex)$) -- ($(q1.north) + (0, 1ex) + (0, 2em)$) node[midway](il){};
    \draw ($(q2.north) + (0, 1ex)$) -- ($(q2.north) + (0, 1ex) + (0, 2em)$) node[midway](ir){};
    \draw[<->] (il) -- (ir) node[midway,above]{Inside};

    \draw ($(s.south) + (0, -1ex)$) -- ($(s.south) + (0, -1ex) + (0, -2em)$) node[midway](oll){};
    \draw ($(q1.south) + (0, -1ex)$) -- ($(q1.south) + (0, -1ex) + (0, -2em)$) node[midway](olr){};
    \draw[<->] (oll) -- (olr) node[midway](ol){};

    \draw ($(q2.south) + (0, -1ex)$) -- ($(q2.south) + (0, -1ex) + (0, -2em)$) node[midway](orl){};
    \draw ($(f.south) + (0, -1ex)$) -- ($(f.south) + (0, -1ex) + (0, -2em)$) node[midway](orr){};
    \draw[<->] (orl) -- (orr) node[midway](or){};

    \draw [decorate,decoration={brace,mirror,amplitude=1em}] ($(ol.south) + (0, -1em)$) -- ($(or.south) + (0, -1em)$) node[midway,yshift=-1.5em] {Outside};
  \end{tikzpicture}
  \caption{Inside and outside weights on a shortest path}
  \label{fig:inside-outside-visual}
\end{figure}
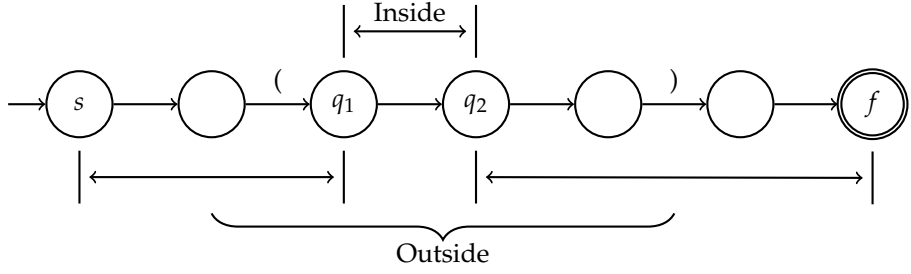

We call $\beta(q_1 \leadsto q_2)$ the \emph{outside weight} of $q_1
\leadsto q_2$, because on the shortest accepting path going through
$q_1 \leadsto q_2$, $\beta(q_1 \leadsto q_2)$ is the weight of the
partial path ``outside'' of $q_1 \leadsto q_2$, as illustrated in
Figure \ref{fig:inside-outside-visual}. This can be easily computed by
applying the Scan and Complete rules in reverse, starting from the
goal after the inside weights have been computed. See Figure
\ref{fig:outside-alg} for a simple algorithm. Very similar to the
inside algorithm in Figure \ref{fig:inside-alg}, we use agenda-based
reasoning, but with the goal item as the starting point (lines
5-6). Then lines 7-20 try to propagate the estimates to inner items by
applying inference rules in reverse.

\begin{figure}
  \centering
  \begin{boxedminipage}{0.9\textwidth}
    \begin{algorithmic}[5]
    \Function {$Outside$}{}
      \State {$\alpha \gets \text{ the inside weights from } Inside$}
      \State {$\beta \gets \text{ empty map}$}
      \State {$Q \gets \text{ empty queue}$}
      \State {$\beta[s \leadsto f] \gets \overline{1}$}
      \State {$Push(s \leadsto f, Q)$}
      \While {$Q$ is not empty}
        \State {$q_1 \leadsto q_2 \gets Pop(Q)$}
        \State {$u \gets \beta[q_1 \leadsto q_2]$}
        \ForAll {incoming $e$ of $q_2$}
          \If {$i[e] \in \Sigma \cup \{ \epsilon \}$} \Comment{Scan in reverse}
            \State{$Relax(q_1 \leadsto p[e], u \otimes w[e])$}
          \ElsIf {$i[e] \in \hat{\Pi}$} \Comment{Complete in reverse}
            \ForAll {$e'$ such that $i[e'] = \hat{i}[e]$ and $q_1 \leadsto p[e']$ and $n[e'] \leadsto p[e]$ both in $\alpha$}
              \State{$Relax(q_1 \leadsto p[e'], u \otimes w[e'] \otimes \alpha[n[e'] \leadsto p[e]] \otimes w[e])$}
              \State{$Relax(n[e'] \leadsto p[e], u \otimes w[e'] \otimes \alpha[q_1 \leadsto p[e']] \otimes w[e])$}
            \EndFor
          \EndIf
        \EndFor
      \EndWhile
    \EndFunction
      \State {}
    \Function {$Relax$}{$q_1 \leadsto q_2, w$}
      \If {$q_1 \leadsto q_2$ is in $\beta$}
        \State {$u \gets \beta[q_1 \leadsto q_2] \oplus w$}
        \If {$u \neq \beta[q_1 \leadsto q_2]$}
          \State {$\beta[q_1 \leadsto q_2] \gets u$}
          \State {$Push(q_1 \leadsto q_2, Q)$ if $q_1 \leadsto q_2$ not already in $Q$}
        \EndIf
      \Else
        \State {$\beta[q_1 \leadsto q_2] \gets u$}
        \State {$Push(q_1 \leadsto q_2, Q)$ if $q_1 \leadsto q_2$ not already in $Q$}
      \EndIf
    \EndFunction
    \end{algorithmic}
  \end{boxedminipage}
  \caption{A simple Outside algorithm}
  \label{fig:outside-alg}
\end{figure}

\subsection{An inexact heuristic and its problems}

The above heuristic is very effective in the search because the
outside weight gives an exact estimate. However, pre-computation of
the outside weight requires two passes traversing the automaton. A
natural question is whether there is an inexact heuristic, yet still
monotonic and admissible, which takes less time to compute.

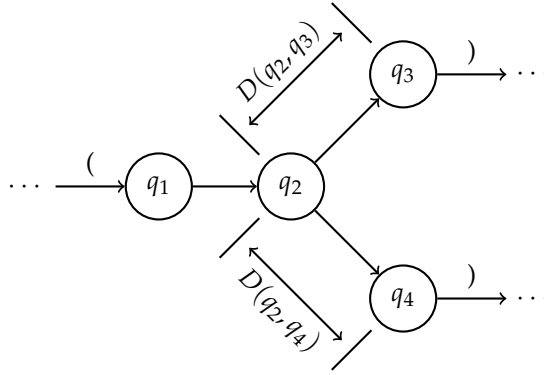
\begin{figure}
  \centering
  \begin{tikzpicture}[node distance=5em]
    \node(q0)[right of=s]{$\cdots$};
    \node[state](q1)[right of=q0]{$q_1$};
    \node[state](q2)[right of=q1]{$q_2$};
    \node[state](q3)[above right of=q2,node distance=6em]{$q_3$};
    \node[state](q4)[below right of=q2,node distance=6em]{$q_4$};
    \node(q5)[right of=q3]{$\cdots$};
    \node(q6)[right of=q4]{$\cdots$};

    \path[->]
    (q0) edge node[auto] {(} (q1)
    (q1) edge (q2)
    (q2) edge (q3)
    (q2) edge (q4)
    (q3) edge node[auto]{)} (q5)
    (q4) edge node[auto]{)} (q6);

    \draw ($(q2) + (-1.2em, 1.2em)$) -- ($(q2) + (-1.2em, 1.2em) + (-1.5em, 1.5em)$) node[midway](tl){};
    \draw ($(q3) + (-1.2em, 1.2em)$) -- ($(q3) + (-1.2em, 1.2em) + (-1.5em, 1.5em)$) node[midway](tr){};
    \draw[<->] (tl) -- (tr) node[midway,sloped,above]{$D(q_2, q_3)$};

    \draw ($(q2) + (-1.2em, -1.2em)$) -- ($(q2) + (-1.2em, -1.2em) + (-1.5em, -1.5em)$) node[midway](bl){};
    \draw ($(q4) + (-1.2em, -1.2em)$) -- ($(q4) + (-1.2em, -1.2em) + (-1.5em, -1.5em)$) node[midway](br){};
    \draw[<->] (bl) -- (br) node[midway,sloped,below]{$D(q_2, q_4)$};

  \end{tikzpicture}
  \caption{$\gamma(q_1 \leadsto q_2) = D(q_2, q_3) \oplus D(q_2, q_4)$ is the shortest distance from $q_2$ to an ``exit''}
  \label{fig:reverse-inside-visual}
\end{figure}

\begin{figure}
  \centering
  \begin{tikzpicture}[node distance=5em]
    \node(q0)[right of=s]{$\cdots$};
    \node[state](q1)[right of=q0]{$q_1$};
    \node[state](q2)[right of=q1]{$q_2$};
    \node[state](q3)[above right of=q2,node distance=6em]{$q_3$};
    \node[state](q4)[below right of=q2,node distance=6em]{$q_4$};
    \node(q5)[right of=q3]{$\cdots$};
    \node(q6)[right of=q4]{$\cdots$};

    \path[<-]
    (q0) edge node[auto] {(} (q1)
    (q1) edge (q2)
    (q2) edge (q3)
    (q2) edge (q4)
    (q3) edge node[auto]{)} (q5)
    (q4) edge node[auto]{)} (q6);

    \draw ($(q2) + (-1.2em, 1.2em)$) -- ($(q2) + (-1.2em, 1.2em) + (-1.5em, 1.5em)$) node[midway](tl){};
    \draw ($(q3) + (-1.2em, 1.2em)$) -- ($(q3) + (-1.2em, 1.2em) + (-1.5em, 1.5em)$) node[midway](tr){};
    \draw[<->] (tl) -- (tr) node[midway,sloped,above]{$\alpha(q_3 \leadsto q_2)$};

    \draw ($(q2) + (-1.2em, -1.2em)$) -- ($(q2) + (-1.2em, -1.2em) + (-1.5em, -1.5em)$) node[midway](bl){};
    \draw ($(q4) + (-1.2em, -1.2em)$) -- ($(q4) + (-1.2em, -1.2em) + (-1.5em, -1.5em)$) node[midway](br){};
    \draw[<->] (bl) -- (br) node[midway,sloped,below]{$\alpha(q_4 \leadsto q_2)$};

  \end{tikzpicture}
  \caption{The reversed WPDA of Figure \ref{fig:reverse-inside-visual}}
  \label{fig:reversed-pda}
\end{figure}
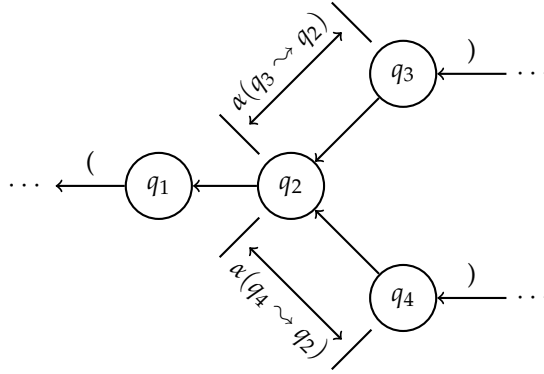

When the multiplication does not decrease the weight,\footnote{For
  example, the tropical semiring with only non-negative weights in the
  setting of the classical shortest path problem on a graph, where
  real valued weights are summed within the path and the minimum is
  taken (i.e. use $+$ as $\otimes$ and $\min$ as $\oplus$).} one may
use the weight of only part of the final shortest accepting path as an
estimate. This can produce a heuristic that is less expensive to
compute, possibly at the cost of increasing the search time. In
particular, define $D(q_1, q_2)$ to be the shortest distance between
any pair of states $q_1$ and $q_2$, and $\gamma(q_1 \leadsto q_2) =
\bigoplus_{q_3} D(q_2, q_3)$, where the summation is over all states
reachable from $q_2$ that have a closing parenthesis or simply $f$
when $q_1$ is $s$ (call such a state an \emph{exiting} state, in
contrast with an \emph{entering} state). $\gamma(q_1 \leadsto q_2)$ is
roughly how far away $q_1 \leadsto q_2$ is to a pair of immediate
enclosing parenthesis.\footnote{This is only a rough estimate since
  there may not be a opening parenthesis going to $q_1$ that matches
  the closing parenthesis of the selected exitting state. However, the
  actual shortest distance is never smaller than this, which means the
  estimate is still admissible.} For example, in Figure
\ref{fig:reverse-inside-visual}, $\gamma(q_1 \leadsto q_2) = D(q_2,
q_3) \oplus D(q_2, q_4)$ is the shortest distance from $q_2$ to
exitting states $q_3$ and $q_4$. All the relevant values of $D$ are in
fact the inside weights of the reversed WPDA of $M$ (for example, see
Figure \ref{fig:reversed-pda}),\footnote{That is, reverse the
  direction of transitions; swap $s$ and $f$; and swap $\Pi$ and
  $\hat{\Pi}$.} therefore we call it the \emph{reverse inside weight}.

Then, we can define the following heuristic,
\[
H_2(q_1 \leadsto q_2 : u) = u \otimes \gamma(q_1 \leadsto q_2)
\]
This gives us the weight of the shortest path starting at the induced
path of $q_1 \leadsto q_2 : u$ to any exiting state, which may be a
part of an accepting path. It is trivially admissible because
$\gamma(s \leadsto f) = D(f, f) = \overline{1}$. When multiplication
does not decrease the weight, the weight of part of a path is always
smaller than or equal to the weight of the whole path. Therefore, the
heuristic is monotonic. Unlike the outside weight, the semiring does
not need to be commutative for this heuristic to be well-defined.

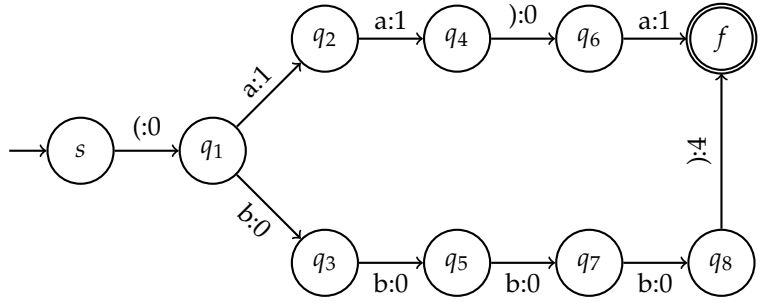
\begin{figure}
  \centering
  \begin{tikzpicture}[node distance=5em]
    \node[state,initial,initial text={}](s){$s$};
    \node[state](q1)[right of=s]{$q_1$};
    \node[state](q2)[above right of=q1,node distance=6em]{$q_2$};
    \node[state](q3)[below right of=q1,node distance=6em]{$q_3$};
    \node[state](q4)[right of=q2]{$q_4$};
    \node[state](q5)[right of=q3]{$q_5$};
    \node[state](q6)[right of=q4]{$q_6$};
    \node[state](q7)[right of=q5]{$q_7$};
    \node[state](q8)[right of=q7]{$q_8$};
    \node[state,accepting](f)[right of=q6]{$f$};

    \path[->]
    (s) edge node[above]{(:0} (q1)
    (q1) edge node[sloped,above]{a:1} (q2)
    (q1) edge node[sloped,below]{b:0} (q3)
    (q2) edge node[above]{a:1} (q4)
    (q4) edge node[above]{):0} (q6)
    (q6) edge node[above]{a:1} (f)
    (q3) edge node[below]{b:0} (q5)
    (q5) edge node[below]{b:0} (q7)
    (q7) edge node[below]{b:0} (q8)
    (q8) edge node[sloped,above]{):4} (f)
    ;
  \end{tikzpicture}
  \caption{A problematic WPDA for $H_2$, weights are in the tropical semiring}
  \label{fig:problem-of-h2}
\end{figure}

\begin{table}
  \centering
  \begin{tabular}{c|cc}
    Item & $\beta$ & $\gamma$ \\
    \hline
    $s \leadsto s$ & 3 & 3 \\
    $s \leadsto q_6$ & 1 & 1 \\
    $s \leadsto f$ & 0 & 0 \\
    $q_1 \leadsto q_1$ & 3 & 0 \\
    $q_1 \leadsto q_2$ & 2 & 1 \\
    $q_1 \leadsto q_3$ & 4 & 0 \\
    $q_1 \leadsto q_4$ & 1 & 0 \\
    $q_1 \leadsto q_5$ & 4 & 0 \\
    $q_1 \leadsto q_7$ & 4 & 0 \\
    $q_1 \leadsto q_8$ & 4 & 0
  \end{tabular}
  \caption{}
  \label{tab:h-values}
\end{table}

Unfortunately, though we now spend less time in pre-computing the
heuristic, the actual A* search usually ends up taking much longer
because of the inexactness. To see why, notice that any instantiation
of an item with a weight smaller than the shortest accepting path has
to be visited, even if it will only be used in an accepting path far
longer than the $k$ shortest ones. For example, consider the WPDA in
Figure \ref{fig:problem-of-h2}, with the relevant values of $\beta$
and $\gamma$ listed in Table \ref{tab:h-values}. When using $H_1$, the
following are instantiated before reaching the 1-shortest path:
\[
\begin{array}{ll}
  s \leadsto s : 0 & (H_1 = 0 + 3 = 3, Q = \{q_1 \leadsto q_1 : 0\}) \\
  q_1 \leadsto q_1 : 0 & (H_1 = 0 + 3 = 3, Q = \{q_1 \leadsto q_2 : 1, q_1 \leadsto q_3 : 0\}) \\
  q_1 \leadsto q_2 : 1 & (H_1 = 1 + 2 = 3, Q = \{q_1 \leadsto q_4 : 2, q_1 \leadsto q_3 : 0\}) \\
  q_1 \leadsto q_4 : 2 & (H_1 = 2 + 1 = 3, Q = \{s \leadsto q_6 : 2, q_1 \leadsto q_3 : 0\}) \\
  s \leadsto q_6 : 2 & (H_1 = 2 + 1 = 3, Q = \{s \leadsto f : 3, q_1 \leadsto q_3 : 0\}) \\
  s \leadsto f : 3 & (H_1 = 3 + 0 = 3, Q = \{q_1 \leadsto q_3 : 0\})
\end{array}
\]
However, when $H_2$ is used, the following are instantiated before the 1-shortest,
\[
\begin{array}{ll}
  q_1 \leadsto q_1 : 0 & (H_2 = 0 + 0 = 0, Q =\{q_1 \leadsto q_3 : 0, q_1 \leadsto q_2 : 1, s \leadsto s : 0\}) \\
  q_1 \leadsto q_3 : 0 & (H_2 = 0 + 0 = 0, Q =\{q_1 \leadsto q_5 : 0, q_1 \leadsto q_2 : 1, s \leadsto s : 0\}) \\
  q_1 \leadsto q_5 : 0 & (H_2 = 0 + 0 = 0, Q =\{q_1 \leadsto q_7 : 0, q_1 \leadsto q_2 : 1, s \leadsto s : 0\}) \\
  q_1 \leadsto q_7 : 0 & (H_2 = 0 + 0 = 0, Q =\{q_1 \leadsto q_8 : 0, q_1 \leadsto q_2 : 1, s \leadsto s : 0\}) \\
  q_1 \leadsto q_8 : 0 & (H_2 = 0 + 0 = 0, Q =\{q_1 \leadsto q_2 : 1, s \leadsto s : 0, s \leadsto f : 4\}) \\
  q_1 \leadsto q_2 : 1 & (H_2 = 1 + 1 = 2, Q =\{q_1 \leadsto q_4 : 2, s \leadsto s : 0, s \leadsto f : 4\}) \\
  q_1 \leadsto q_4 : 2 & (H_2 = 2 + 0 = 2, Q =\{s \leadsto s : 0, s \leadsto f : 4\}) \\
  s \leadsto s : 0 & (H_2 = 0 + 3 = 3, Q =\{s \leadsto q_6 : 0, s \leadsto f : 4\}) \\
  s \leadsto q_6 : 2 & (H_2 = 2 + 1 = 3, Q =\{s \leadsto f : 3, s \leadsto f : 4\}) \\
  s \leadsto f : 3 & (H_2 = 3 + 0 = 3, Q =\{s \leadsto f : 4\})
\end{array}
\]
$H_2$ ends up visiting more instantiations along the path from $q_1$
to $q_8$ because that path is shorter in the scope of the enclosing
parentheses. The following closing parenthesis completely flips the
position, but this is some information $H_1$ ``knows'' while $H_2$
does not. In practice, we find this happens so frequently that $H_2$
fails to output the shortest path within a reasonable amount of time.

Another problem with $H_2$ is that when multiplication may increase
the weight (for example, in the tropical semiring with negative
weights, which is commonly used in applications such machine
translation), the heuristic is no longer monotonic.

\section{Algorithm 2}
\label{sec:algorithm-2}

We can adopt a new search strategy to address the problems with
$H_2$. Before describing the algorithm, we take a brief excursion to
introduce a technique from \citet{huang-chiang:2005:IWPT}. Consider
the following problem:
\begin{quote}
  Let $A$ and $B$ be two (possibly infinite) ordered sequence of real
  numbers (i.e. for any $i$, $A_i \leq A_{i+1}$ and $B_i \leq
  B_{i+1}$). Find the $k$ smallest elements in $A \times B$, ordered
  by the sum of the pair.
\end{quote}

\begin{figure}
  \centering
  \begin{boxedminipage}{0.9\textwidth}
    \begin{algorithmic}[5]
      \State {$Q \gets \text{ empty min-priority queue}$}
      \State {$Push(\langle 1, 1 \rangle, Q)$ with priority $A_1 + B_1$}
      \While {$Q$ is not empty}
        \State {$\langle i, j \rangle \gets Pop(Q)$}
        \State {Output $\langle A_i, B_j \rangle$}
        \If {$\langle i+1, j \rangle$ not already in $Q$}
          \State {$Push(\langle i+1, j \rangle, Q)$ with priority $A_{i+1} + B_{j}$}
        \EndIf
        \If {$\langle i, j+1 \rangle$ not already in $Q$}
          \State {$Push(\langle i, j+1 \rangle, Q)$ with priority $A_{i} + B_{j+1}$}
        \EndIf
      \EndWhile
    \end{algorithmic}
  \end{boxedminipage}
  \caption{}
  \label{fig:lazy-mult-alg}
\end{figure}

For example, when $A$ is $\{0, 2, 2\}$ and $B$ is $\{1, 2, 4\}$, the 3
smallest elements are $\{\langle 0, 1 \rangle, \langle 0, 2 \rangle,
\langle 2, 1 \rangle \}$. A naive solution is to compute the first $k$
elements in both $A$ and $B$ then sort all the $k^2$ combinations. The
technique from \citet{huang-chiang:2005:IWPT}, described in Figure
\ref{fig:lazy-mult-alg}, visits at most $2k$ combinations and usually
a lot fewer in practice. The key insight is that there is no need to
explore $\langle A_{i+1}, B_j \rangle$ or $\langle A_i, B_{j+1}
\rangle$ before $\langle A_i, B_j \rangle$ is popped because both of
them are guaranteed to be sub-optimal compared with $\langle A_i, B_j
\rangle$. When computing elements in $A$ and $B$ is expensive, this
technique is substantially faster than the naive solution.

The same idea can be applied in our problem. For any pair of entering
and exiting states $\langle p, q \rangle$, let $G_{pq}$ be the
sequence of balanced paths from $p$ to $q$ ordered by their weight and
let $G_{pq}^i$ be the $i$-th path. Following similar reasoning, we
know there is no need to compute the actual value of $G_{pq}^{i+1}$
before $G_{pq}^i$ is ever used as part of a larger path, in search of
the $k$ shortest accepting path. Furthermore, $G_{pq}$ can be
incrementally computed, as we show next in Figure
\ref{fig:lazy-expand-alg}.

\begin{figure}
  \centering
  \begin{boxedminipage}{0.9\textwidth}
    \begin{algorithmic}[5]
    \Function{$FindKth$}{$p, q, k$} \Comment{Finds the $k$-th element of $G_{pq}$}
      \If {the result has been cached}
        \State {\Return {the cached result}}
      \EndIf
      \State {$S$ is a global variable storing proven items, initialized as empty outside the function}
      \State {$Q_{pq}$ is a min-priority queue, initialized as empty outside the function}

      \If {$p \leadsto p$ not in $S$} \Comment{First time called}
        \State {$Push(p \leadsto p : \overline{1}, Q_{pq})$ with priority $D(p, q)$}
      \EndIf

      \While {$Q_{pq}$ is not empty}
        \If {top of $Q_{pq}$ is proven via Scan}
          \State {$p \leadsto r : u \gets Pop(Q_{pq})$}
        \Else \Comment {via Complete; further pushing is needed}
          \State {$\langle p \leadsto r : u, v, e, e', j \rangle \gets Pop(Q_{pq})$}
          \State {$n[e] \leadsto p[e'] : w \gets FindKth(n[e], p[e'], j+1)$}
          \State {$h \gets v \otimes w[e] \otimes w \otimes w[e'] \otimes D(n[e'], q)$}
          \If {$h \neq \overline{0}$}
            \State {$Push(\langle p \leadsto r : v \otimes w[e] \otimes w \otimes w[e'], v, e, e', j+1 \rangle, Q_{pq})$ with priority $h$} \Comment {Store information for further pushing in the future}
          \EndIf
        \EndIf

        \State {Add $p \leadsto r : u$ to $S$}
        \ForAll {transition $e$ such that $p[e] = r$}
          \If {$i[e] \in \Sigma \cup \{\epsilon\}$} \Comment {Scan}
            \State {$h \gets u \otimes w[e] \otimes D(n[e], q)$}
            \If {$h \neq \overline{0}$}
              \State {$Push(p \leadsto n[e] : u \otimes w[e], Q_{pq})$ with priority $h$}
             \EndIf
          \ElsIf {$i[e] \in \Pi$} \Comment {Complete; as the left antecedent}
            \ForAll {transition $e'$ such that $i[e'] = \hat{i}[e]$ and $D(n[e], p[e']) \neq \overline{0}$}
              \State {$n[e] \leadsto p[e'] : v \gets FindKth(n[e], p[e'], 1)$}
              \State {$h \gets u \otimes w[e] \otimes v \otimes w[e'] \otimes D(n[e'], q)$}
              \If {$h \neq \overline{0}$}
                \State {$Push(\langle p \leadsto n[e'] : u \otimes w[e] \otimes v \otimes w[e'], u, e, e', 1 \rangle, Q_{pq})$ with priority $h$} \Comment {Store information for further pushing in the future}
              \EndIf
            \EndFor
          \EndIf
        \EndFor

        \If {$p \leadsto r : w$ is a goal item}
          \State Cache $p \leadsto r : w$, then return $p \leadsto r : w$
        \EndIf
      \EndWhile
    \EndFunction
    \end{algorithmic}
  \end{boxedminipage}
  \caption{Algorithm 2}
  \label{fig:lazy-expand-alg}
\end{figure}

The algorithm operates as follows. First of all, instead of having a
single priority queue, now for every relevant $G_{pq}$, we have a
corresponding priority queue $Q_{pq}$. $Q_{pq}$ is only responsible
for finding the intermediate ``goal'', i.e. balanced paths from $p$ to
$q$, in increasing order of their weights. Further, only items of the
form $p \leadsto r$ are pushed into $Q_{pq}$, which allows us to use
the following heuristic that only requires the reverse inside weights,
\[
H_{pq}(p \leadsto r : u) = u \otimes D(r, q)
\]
Items are then proved in a top-down fashion, starting with
$G_{sf}$. The search process can be described recursively (Figure
\ref{fig:lazy-expand-alg}). Let the sequence in consideration be
$G_{pq}$,
\begin{itemize}
\item If there is no balanced path from $p$ to $q$ using any
  parenthesis, all proofs only involve the Scan rule. As a result,
  $G_{pq}$ can be incrementally computed without consulting any other
  sequence (lines 23-27).
\item Otherwise, let $e$ and $e'$ be the pair of parentheses
  encountered during the search. Simply query $G_{n[e]p[e']}$ to get
  the shortest path (line 30), and only use the $(k+1)$-th shortest
  path after an instantiation proved with the $k$-th one is popped
  (lines 14-18).
\end{itemize}

Though omitted in Figure \ref{fig:lazy-expand-alg} for a simpler
presentation, a further optimization is essential to achieve the
desired efficiency. Observe in the second case above, that the exact
knowledge of the shortest path from $n[e]$ to $p[e']$ is not required
until an item proved using that path is popped. Therefore, instead of
directly calling $FindKth(n[e], p[e'], 1)$, one can query $D(n[e],
p[e'])$ to get the shortest distance. This is sufficient to compute
the priority and ``promise'' an actual proof, which will be realized
once the item is popped. To distinguish actual instantiations from
those with a promise, we denote $q_1 \sim q_2 : u$ as an instantiation
where the last step is based on a promise.

\begin{table}
  \centering
  \begin{tabular}{c|c}
    Pair of states & D \\
    \hline
    $s, f$ & 3 \\
    $q_1, q_4$ & 2 \\
    $q_2, q_4$ & 1 \\
    $q_4, q_4$ & 0 \\
    $q_1, q_8$ & 0 \\
    $q_3, q_8$ & 0 \\
    $q_5, q_8$ & 0 \\
    $q_7, q_8$ & 0 \\
    $q_8, q_8$ & 0
  \end{tabular}
  \caption{}
  \label{tab:d-values}
\end{table}
To see the new algorithm at work, consider again the WPDA in Figure
\ref{fig:problem-of-h2}. Relevant values of $D$ are listed in Table
\ref{tab:d-values}. Then the following are instantiated before
reaching the 1-shortest path,
\begin{center} \small
  \begin{tabular}{ll|ll}
    \multicolumn{2}{c|}{$G_{sf}$} & \multicolumn{2}{c}{$G_{q_1q_4}$} \\
    $s \leadsto s : 0$ & $(H_{sf} = 0 + 3 = 3, Q_{sf} = \{s \sim q_6 : 2, s \sim f : 4\})$ & \\
    $s \sim q_6 : 2$ & $(H_{sf} = 2 + 1 = 3, Q_{sf} = \{s \sim f : 4\})$ & \\
    && $q_1 \leadsto q_1 : 0$ & $(H_{q_1q_4} = 0 + 2 = 2, Q_{q_1q_4} = \{q_1 \leadsto q_2 : 1\})$ \\
    && $q_1 \leadsto q_2 : 1$ & $(H_{q_1q_4} = 1 + 1 = 2, Q_{q_1q_4} = \{q_1 \leadsto q_4 : 2\})$ \\
    && $q_1 \leadsto q_4 : 2$ & $(H_{q_1q_4} = 2 + 0 = 2, Q_{q_1q_4} = \{\})$ \\
    $s \leadsto q_6 : 2$ & $(H_{sf} = 2 + 1 = 3, Q_{sf} = \{s \leadsto f : 3, s \sim f : 4\})$ & \\
    $s \leadsto f : 3$ & $(H_{sf} = 3 + 0 = 3, Q_{sf} = \{s \sim f : 4\})$ &
  \end{tabular}
\end{center}
Notice no item is ever instantiated from $G_{q_1q_8}$, which is
exactly the desired result.

Another benefit of grouping the search by the intermediate ``goals''
is there is not any special requirement on the semiring ---
multiplication neither has to be commutative nor non-decreasing.

\section{Experimental Results}
\label{sec:experiment}

\begin{figure}
  \centering
  \includegraphics[width=0.6\linewidth]{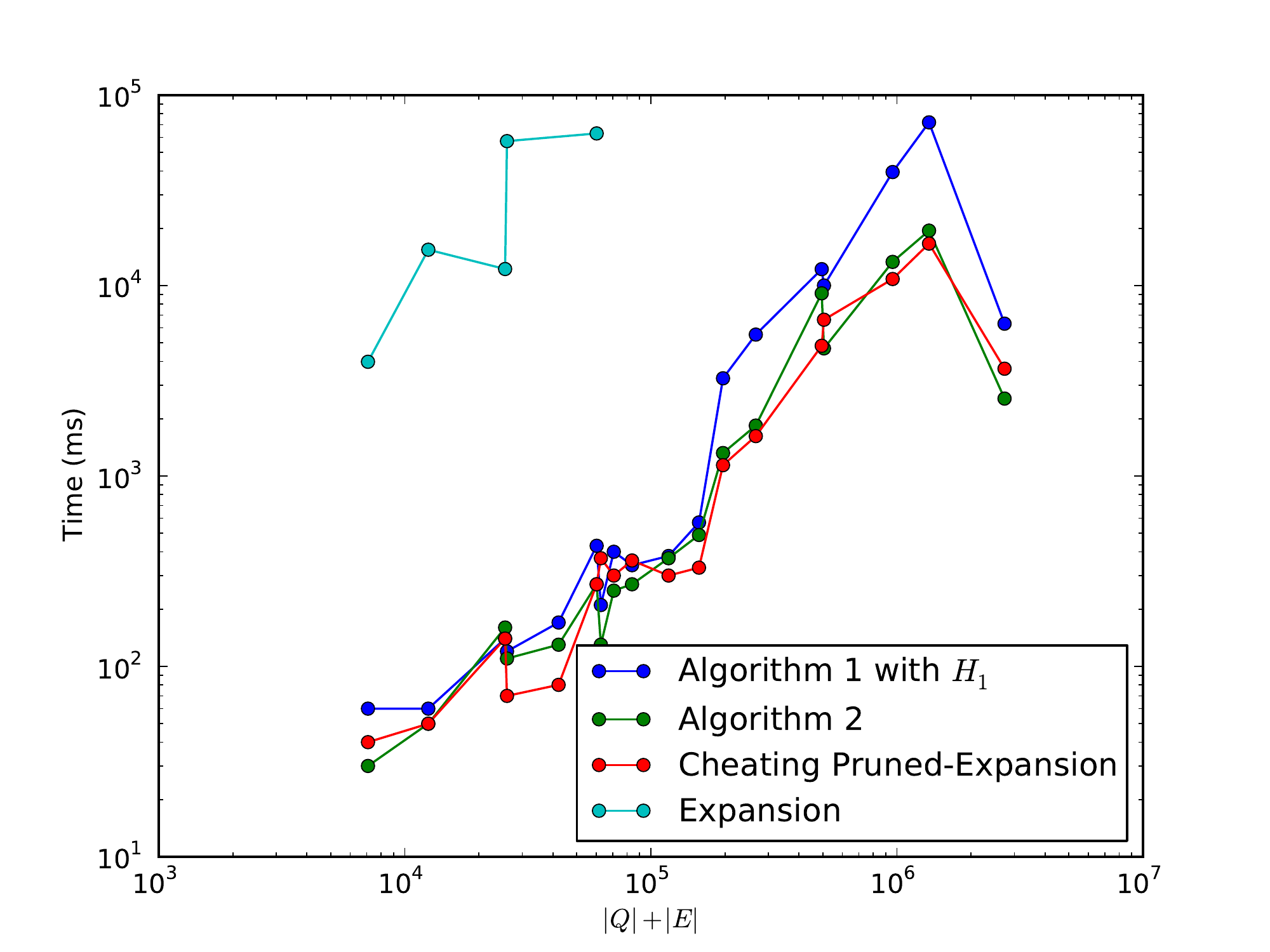}
  \caption{Timing on WPDAs with various sizes, $k = 1000$}
  \label{fig:size-time}
\end{figure}

\begin{figure}
  \centering
  \includegraphics[width=0.6\linewidth]{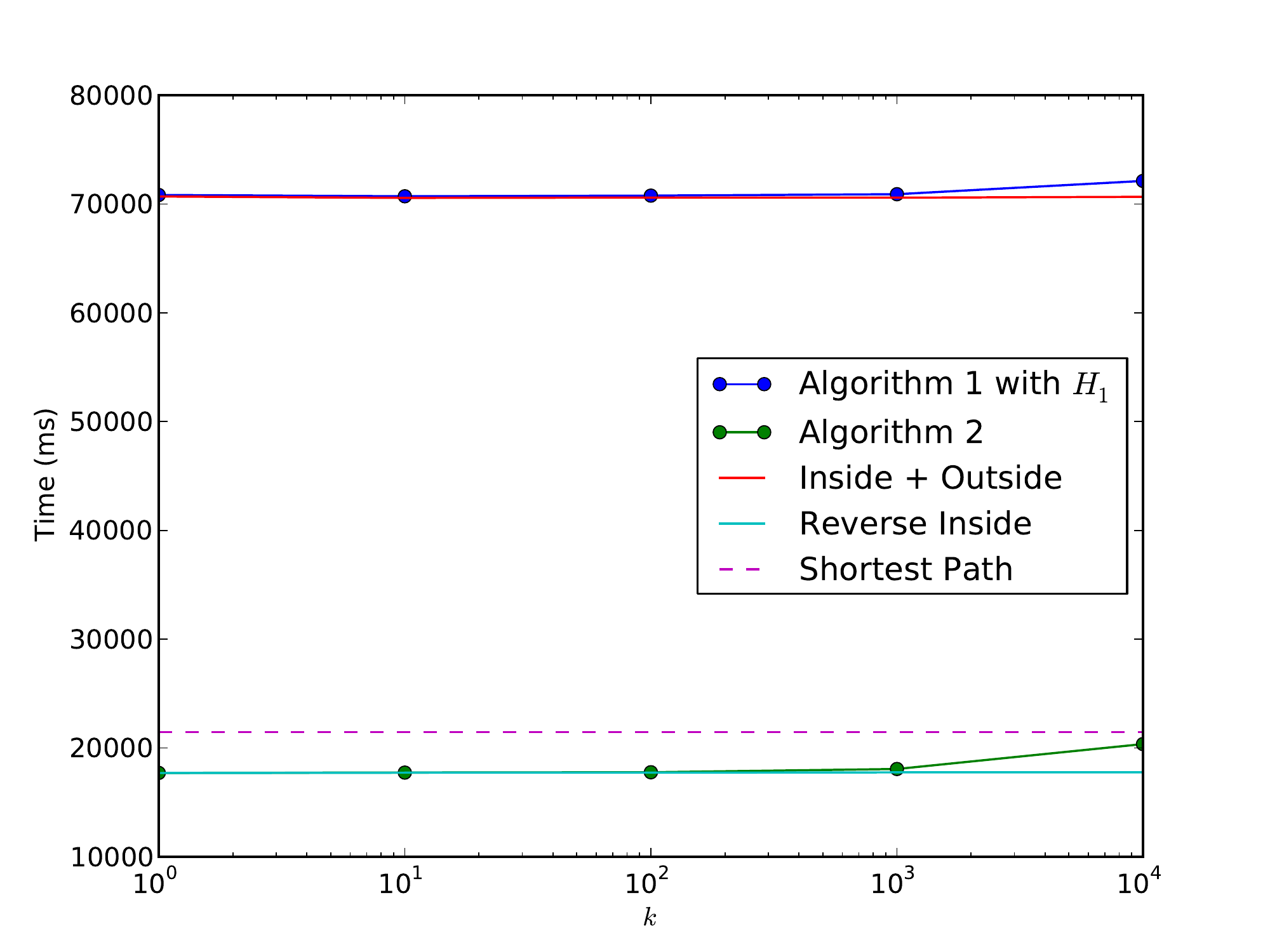}
  \caption{Timing on a WPDA with 398347 states and 951889 transitions}
  \label{fig:k-time}
\end{figure}

We tested our algorithms on WPDAs generated from the machine
translation system described in
\citet{iglesias_hierarchical_2011}. Figure \ref{fig:size-time}
compares the running time of the two algorithms with two previous
approaches (expansion and pruned-expansion with orcale threshold) in
finding the 1000 shortest paths on sample WPDAs with various
sizes. Due to the exponential time complexity of WPDA expansion, the
expansion baseline is only able to finish within our time and memory
limit on the 5 smallest sample inputs.\footnote{2 CPU hours; 4 GB of
  memory.} For the pruned-expansion approach, we pick the oracle
threshold (the exact weight difference between the shortest path and
the 1000th shortest one) for each sample.

Both of our algorithms are significantly faster than the expansion
baseline, and their performance is comparable on smaller input. But as
the size of the WPDA grows, the advantage of the single pass
pre-computation of Algorithm 2 becomes clear, resulting in a very
large time improvement in this case.

The performance of Algorithm 2 is close to the pruned-expansion's
oracle best case in almost all sample inputs. However, it is worth
noting that the perfect threshold varies significantly between samples
--- even for those generated from the same system using different
inputs, the factor of the perfect threshold relative to the weight of
the shortest path can vary from 0.35\% to 160\% while the median is
7\%. This justifies our previous claim about the difficulty in picking
an appropriate threshold.

Figure \ref{fig:k-time} breaks down the running time of our algorithms
on a large WPDA. Both of them spend most of their time on
pre-computing the heuristics and the actual search takes very little
time even with $k$ as large as 10000.

\section{Conclusion}
\label{sec:conclusion}

In this paper, we developed two algorithms for finding $k$ shortest
paths of a WPDA. Previously, there were two approaches to this
problem. The expansion approach expands the WPDA into an equivalent
WFSA, which requires exponential time and space, and then finds the
$k$ shortest paths of the WFSA. Another pruned-expansion approach
expands the WPDA into a WFSA with states or transitions not on a path
close enough by a given threshold to the shortest path by weight
removed, and then finds the $k$ shortest paths of the pruned
WFSA. This requires less time and space, but an appropriate threshold
is almost impossible to set.

In contrast, our algorithms do not need any pruning or threshold
picking and give the exact $k$ shortest paths. The experimental
results on real world input show that Algorithm 2 is highly efficient,
adding very little overhead to the shortest distance pre-computation,
whose running time is comparable to the original shortest path
algorithm in \citet{allauzen_pushdown_2012}.

\section*{Acknowledgements}

We would like to thank Gonzalo Iglesias for providing test inputs for
our experiments. This research was supported in part by the BOLT
program of the Defense Advanced Research Projects Agency, Contract
No. HR0011-12-C-0015. Any opinions, findings, conclusions or
recommendations expressed in this paper are those of the authors and
do not necessarily reflect the view of DARPA.

\bibliography{pdakbest.bib}

\appendix
\section{Proof of Properties of $\mathcal{L}_M$}

\begin{thm*}[Soundness]
  Any valid proof of an instantiation $q_1 \leadsto q_2 : u$ in $\mathcal{L}_M$
  induces a balanced path from $q_1$ to $q_2$ with weight $u$ in $M$.
\end{thm*}

\begin{proof}
\begin{description}
\item[Base] An axiom of the form $q_1 \leadsto q_2 : u$ must have $q_1
  = q_2$ and $u = \overline{1}$. The yield of a proof using only the
  axiom is an empty path, thus a balanced path with weight
  $\overline{1}$.
\item[Induction] Assuming proofs with at most $n$ steps satisfy the
  above lemma. For any proof of $q_1 \leadsto q_2 : u$ in $n+1$ steps,
  \begin{itemize}
  \item If the last step uses the Scan rule, then it must be of the
    following form,
    \begin{prooftree}
      \AxiomC{$q_1 \leadsto q_3 : u_1$}
      \RightLabel{$q_3 \xrightarrow{a} q_2 : u_2$}
      \UnaryInfC{$q_1 \leadsto q_2 : u$}
    \end{prooftree}
    where $u_1 \otimes u_2 = u$, $q_3 \xrightarrow{a} q_2 : u_2 \in E$
    and $q_1 \leadsto q_3 : u_1$ is the outcome of some proof in at
    most $n$ steps. Let the induced path of the proof of $q_1 \leadsto
    q_3 : u_1$ be $\pi' = e_1 e_2 \ldots e_m$. The induced path of the
    whole proof is thus $\pi = e_1 e_2 \ldots e_m (q_3 \xrightarrow{a}
    q_2)$. By the induction hypothesis, $\pi'$ is a balanced path from
    $q_1$ to $q_3$ with weight $u_1$. As a result, $\pi$ is also
    balanced because $a \in \Sigma \cup \{ \epsilon \}$ by definition
    of the logic; its weight is $w[\pi] = w[\pi'] \otimes u_2 = u_1
    \otimes u_2 = u$.
  \item If the last step uses the Complete rule, then it must be of the
    following form,
    \begin{prooftree}
      \AxiomC{$q_1 \leadsto q_3 : u_1$}
      \AxiomC{$q_4 \leadsto q_5 : u_3$}
      \RightLabel{$q_3 \xrightarrow{a} q_4 : u_2$, $q_5 \xrightarrow{\hat{a}} q_2 : u_4$}
      \BinaryInfC{$q_1 \leadsto q_2 : u$}
    \end{prooftree}
    where $u_1 \otimes u_2 \otimes u_3 \otimes u_4 = u$, $a \in \Pi$
    is an opening parenthesis, and $\hat{a} \in \hat{\Pi}$ is the
    corresponding closing parenthesis. Similar to the Scan rule case,
    one can prove the induced path is a balanced path from $q_1$ to
    $q_2$ with weight $u$ using the associativity of $\otimes$.
  \end{itemize}
\end{description}
\end{proof}

\begin{thm*}[Completeness]
  Any balanced path from an entering state $q_1$ to some state $q_2$
  with weight $u$ in $M$ has a valid proof of an instantiation $q_1
  \leadsto q_2 : u$ in $\mathcal{L}_M$ whose induced path is that
  path.
\end{thm*}

\begin{proof}
  \begin{description}
  \item[Base] For any empty balanced path from a state $q$ such that
    $q \leadsto q : \overline{1}$ is an axiom, the proof is just the
    axiom itself.
  \item[Induction] Assuming all balanced paths from any entering state of
    at most length $n$ satisfy the above lemma. For any balanced path
    of length $n+1$ from an entering state $\pi = e_1 e_2 \ldots e_{n+1}$
    from $q_1$ to $q_2$ with weight $u$,
    \begin{itemize}
    \item If $e_{n+1}$ is $q_3 \xrightarrow{a} q_2 : u_2$ with $a \in
      \Sigma \cup \{ \epsilon \}$, then $\pi' = e_1 e_2 \ldots e_m$ is
      a balanced of length $n$ and $w[\pi'] \otimes u_2 = u$. By
      induction hypothesis, there exists a proof of $\pi'$ (via item
      $q_1 \leadsto n[e_m] : w[\pi']$). Applying the Scan rule then
      gives a proof of $q_1 \leadsto q_2 : u_1 \otimes u_2 = u$.
    \item If $e_{n+1}$ is $q_3 \xrightarrow{a} q_2 : u_2$ with $a \in
      \hat{\Pi}$, then there must be a $k \leq n$ such that $e_k$
      balances with $e_{n+1}$. Similar to the above case, one can
      prove the item by combining the proof of $e_1 \ldots e_{k-1}$
      and $e_{k+1} \ldots e_m$.
    \item If $e_{n+1}$ is $q_3 \xrightarrow{a} q_2 : u_2$ with $a \in
      \Pi$, the path cannot be balanced.
    \end{itemize}
  \end{description}
\end{proof}

\begin{thm*}[In-ambiguity]
  Any balanced path from an entering state in $M$ has a unique proof
  in $\mathcal{L}_M$.
\end{thm*}

\begin{proof}
  This is very similar to the completeness proof, thus we only give a
  sketch of the proof. First note all empty paths from an entering
  state has a unique proof (if the start state happens to have an
  incoming open-parenthesis transition, we consider the two inducing
  the same axiom). For any longer paths, if the last transition has a
  label from $\Sigma \cup \{ \epsilon \}$ then the last step must be
  using the Scan rule with antecedents with unique proof and that
  particular transition as the side condition; otherwise the last
  transition must have a closing parenthesis, which means a unique
  application of the Complete rule.
\end{proof}

\end{document}